\documentclass[letterpaper]{article} 
\usepackage{aaai2026}
\usepackage{times}  
\usepackage{helvet}  
\usepackage{courier}  
\usepackage[hyphens]{url}  
\usepackage{graphicx} 
\usepackage{natbib}  
\usepackage{caption} 
\usepackage{amssymb} 
\usepackage{graphicx}
\usepackage{amsmath}
\usepackage{booktabs}
\usepackage{url}
\usepackage{multirow}
\usepackage{epstopdf}
\usepackage{adjustbox}
\usepackage{amsthm}
\newtheorem{remark}{Remark}
\newtheorem{lemma}{Lemma}

\newtheorem{theorem}{Theorem}

\usepackage{enumitem}
\usepackage{array}
\usepackage{subfigure}
\usepackage{wasysym}
\usepackage{xcolor}  
\usepackage{tikz}  
\usepackage{fontawesome}  

\usepackage[linesnumbered,ruled]{algorithm2e}

\definecolor{darkred}{RGB}{139, 0, 0}
\definecolor{darkgreen}{RGB}{37, 193, 123}

\usepackage{newfloat}
\usepackage{listings}
\DeclareCaptionStyle{ruled}{labelfont=normalfont,labelsep=colon,strut=off} 
\title{Break the Tie: Learning Cluster-Customized Category Relationships for Categorical Data Clustering}

\author {
    Mingjie Zhao\textsuperscript{\rm 1},
    Zhanpei Huang\textsuperscript{\rm 2},
    Yang Lu\textsuperscript{\rm 3},
    Mengke Li\textsuperscript{\rm 4},\\
    Yiqun Zhang\textsuperscript{\rm 2}\equalcontrib,
    Weifeng Su\textsuperscript{\rm 5},
    Yiu-ming Cheung\textsuperscript{\rm 1}\equalcontrib
}
\affiliations {
    \textsuperscript{\rm 1}Department of Computer Science, Hong Kong Baptist University, Hong Kong SAR, China \\
    \textsuperscript{\rm 2}School of Computer Science and Technology, Guangdong University of Technology, Guangzhou, China\\
    \textsuperscript{\rm 3}School of Informatics, Xiamen University, Xiamen, China\\
    \textsuperscript{\rm 4}College of Computer Science and Software Engineering, Shenzhen University, Shenzhen, China\\
    \textsuperscript{\rm 5}Guangdong Key Lab of AI and Multi-Modal Data Processing, BNU-HKBU United International College, Zhuhai, China\\
    mjzhao@comp.hkbu.edu.hk, 2112405010@mail2.gdut.edu.cn, luyang@xmu.edu.cn, mengkeli@szu.edu.cn, yqzhang@gdut.edu.cn, wfsu@uic.edu.cn, ymc@comp.hkbu.edu.hk
}

\usepackage{bibentry}

\begin{document}

\maketitle

\begin{abstract}
Categorical attributes with qualitative values are ubiquitous in cluster analysis of real datasets. Unlike the Euclidean distance of numerical attributes, the categorical attributes lack well-defined relationships of their possible values (also called categories interchangeably), which hampers the exploration of compact categorical data clusters. Although most attempts are made for developing appropriate distance metrics, they typically assume a fixed topological relationship between categories when learning distance metrics, which limits their adaptability to varying cluster structures and often leads to suboptimal clustering performance. This paper, therefore, \textit{breaks the intrinsic relationship tie} of attribute categories and learns customized distance metrics suitable for flexibly and accurately revealing various cluster distributions. As a result, the fitting ability of the clustering algorithm is significantly enhanced, benefiting from the learnable category relationships. Moreover, the learned category relationships are proved to be Euclidean distance metric-compatible, enabling a seamless extension to mixed datasets that include both numerical and categorical attributes. Comparative experiments on 12 real benchmark datasets with significance tests show the superior clustering accuracy of the proposed method with an average ranking of 1.25, which is significantly higher than the 5.21 ranking of the current best-performing method. Code and extended version with detailed proofs are provided below.
\end{abstract}

\begin{links}
    \link{Code \& Datasets}{https://github.com/ZHAO-Mingjie/SCOF}
\end{links}

\section{Introduction}
\label{sct:intro}

Clustering analysis on real data usually deals with categorical attributes characterized by qualitative values without a well-defined category relationship like the Euclidean distance of numerical data~\cite{intro1, zhang2025adaptive,zhao2022heterogeneous,zhang2025towards}. For instance, the distance between ``lawyer" and ``driver" in an attribute ``Occupation" of a categorical dataset is not universally defined as the distance between 0.5 and 1.2 of numerical data. Since distance metric is the basis for most clustering algorithms, lacking suitable distance metrics may surely prevent the existing clustering techniques designed for numerical data from achieving satisfactory performance on categorical data. Moreover, additionally considering the definition of an appropriate distance metric brings greater challenges to categorical data clustering~\cite{intro3, intro4,feng2025robust}.

\begin{figure}[!t]	
\centerline{\includegraphics[width=1.04\linewidth]{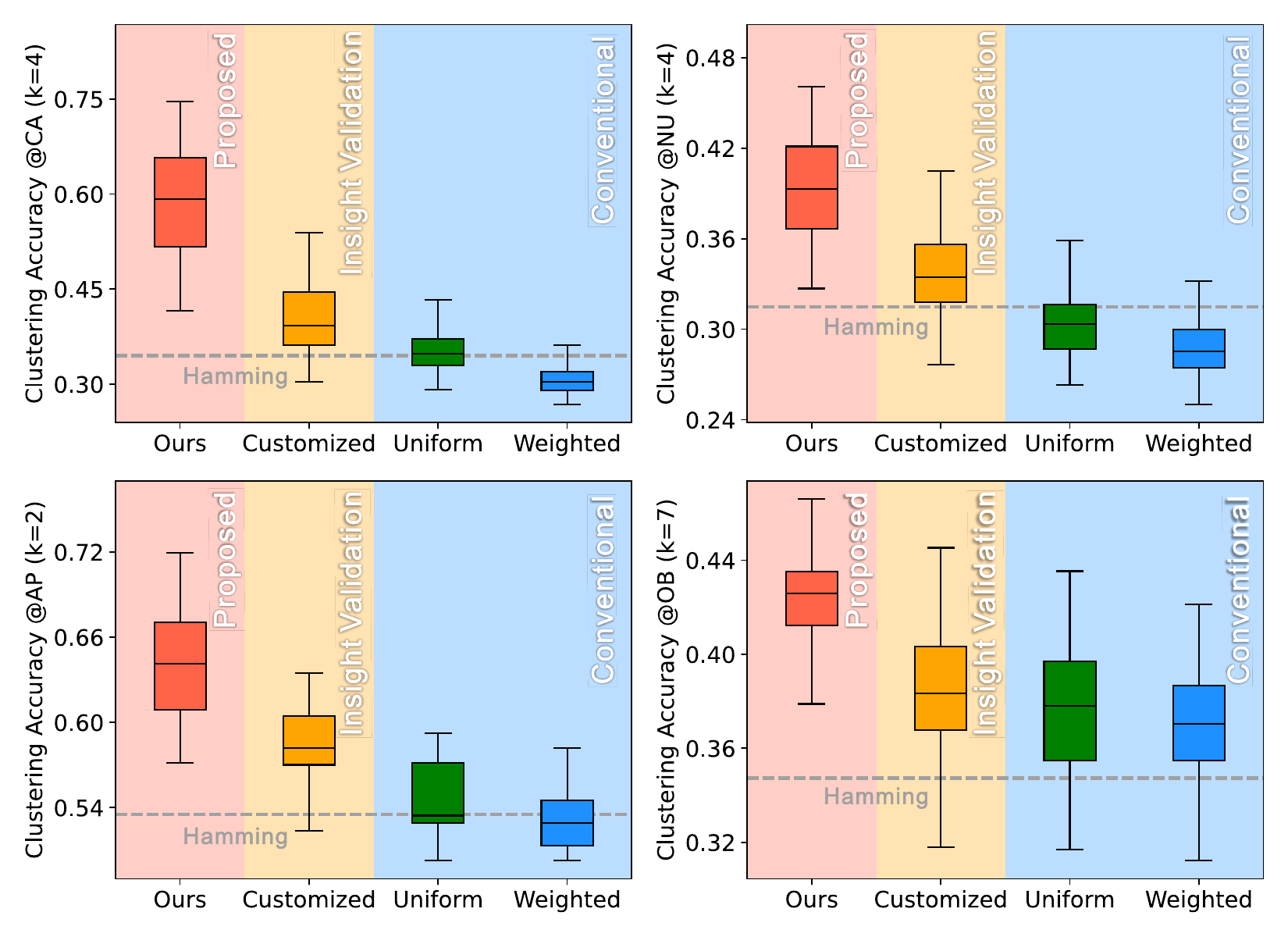}}
\caption{Clustering accuracy of our methods and three different distance measurement strategies using $k$-modes: Customized, Uniform, and Weighted on four datasets CA, NU, AP and TA in \textbf{Table~\ref{tb:statistics}}. Additionally, Hamming is also included as a baseline. $k$ indicates the number of clusters. Our method outperforms other strategies on these datasets. The ``Customized" strategy achieved better results than ``Uniform'' and ``Weight'' because it assigns customized category relationships to different clusters. }	
\label{fig:toy_example}	
\end{figure}

To uncover the implicit distance relationship of the possible values of categorical data~\cite{intro2,intro5,intro10}, some research primarily focuses on developing distance measures based on 
value-mapping techniques~\cite{hdm}, information entropy~\cite{lsm, udm, ebdmjournal}, significance test~\cite{he1sscm,he2interpretable,he3clusterability,he4CDM}, and statistical characteristics~\cite{adm, abdm, cde, cbdm_journal, CDCDR}. However, they are all task-independent and thus incapable of adapting to different clustering tasks across various categorical datasets. To address this limitation, distance learning methods further connect the distance to downstream clustering tasks using approaches such as similarity learning based on the distribution probabilities of attribute values ~\cite{ocil}, kernel space mapping~\cite{untie}, and graph-based learning~\cite{HDC,  ADC, CoForest}. While they achieve competitive clustering performance, the learned category relationships remain uniform for the global cluster distribution. This approach neglects the subspace specificity of clusters, i.e., some samples may exhibit a compact clustering effect in a certain subspace spanned by an attribute subset.

Although some subspace categorical data clustering methods~\cite{wchd, wksc, WOCIL} are proposed considering the attribute-cluster relationship recently, they primarily focus on learning the attributes' importance while still adopting a uniform category relationship among the possible values of each attribute across all the clusters. In fact, the same attribute may exhibit diverse category relationships across different clusters (i.e., conventional part of \textbf{Figure~\ref{fig:toy_example}}). For example, when considering the attribute ``occupation'' ($\mathbf{a}_o$) in a pneumonia patient dataset, the category relationship of $\mathbf{a}_o$'s possible values \{programmer, construction worker, driver, ...\} differs in two diagnosis clusters, one for COVID-19 ($C_{co}$), and one for pneumoconiosis ($C_{pn}$). That is, for cluster $C_{co}$, the occupational difference between ``programmer'' and ``construction worker'' is relatively small, as their chances of contracting the highly contagious COVID-19 are comparable, while for cluster $C_{pn}$, the difference between these two occupations is huge because ``construction worker'''s prolonged exposure to dust-laden environments significantly increases their likelihood of developing pneumoconiosis compared to programmers. Therefore, for categorical data, whose values are inherently more vague compared to numerical data, it is promising to define subspace category relationships for each attribute to better serve the exploration of different clusters.

To validate the effectiveness of the subspace category relationship for categorical data clustering, in \textbf{Figure~\ref{fig:toy_example}}, the $k$-modes~\cite{kmd} clustering accuracy corresponding to the following distance measurement strategies is plotted using box plots: 1) Customized: $k$-modes + random Customized category relationships, i.e., for a given attribute, randomly generate a series of category relationships for different clusters; 2) Uniform: $k$-modes + random Uniform category relationship; 3) Weighted: $k$-modes + Hamming distance + random Weighted attributes per the cluster; Additionally, $k$-modes + Hamming distance is also included as a baseline. The category relationships of the Customized, Uniform, and Weighted are randomly generated and are executed 100 times. It can be seen that: 1) Customized outperforms the others, which demonstrates that clustering-customized category relationship brings better clustering results than uniform category relationships. 2) Uniform clearly surpasses Weighted and Hamming, which indicates Intra-attribute distance relationship modeling outperforms attribute weighting. Therefore, obtaining the optimal subspace category relationship i.e., learning clustering-customized relationship among intra-attribute values for each cluster, is the key to achieving accurate categorical data clustering.

In this paper, a new learning paradigm for categorical data clustering that learns clustering-customized category relationships from cluster (DISC) is proposed. The learning process is no longer constrained to adjusting attribute weights for clusters under uniform and fix category relationships. Instead, category relationship is initially assigned to all clusters, which then evolves into different structures for each cluster during the learning process. This allows the relationships between attribute values to be adaptively captured under a high degree of freedom, ensuring that the learned distance relationships are meaningful and tailored to improve the quality and compactness of the clusters, thus enhancing clustering performance, as can be seen in \textbf{Figure~\ref{fig:toy_example}}. We model the distance relationship of intra-attribute possible values as a fully connected graph, and infer a tree-like category relationship to concisely and exactly reflect the relationship among attribute values. As a result, each attribute has a series of category relationships tailored for different clusters, significantly enhance the fitting ability for cluster optimization. Moreover, it is theoretically and empirically demonstrated that the learned category relationships exhibit strong compatibility with the Euclidean distance, enabling a natural extension of the method to mixed datasets, which include both numerical and categorical attributes, and resulting in enhanced clustering performance. The main contributions of this work are summarized into three-fold:

\begin{itemize}

\item \textbf{Revealing the effect of category relationships in clustering}: The coupling between clustering and category relationships is studied by extending the subspace clustering principles, and it is validated that optimal category relationship is the foundation for accurate clustering.

\item \textbf{Proposing a novel paradigm for category relationships learning}: A novel learning strategy is proposed to modeling and tuning the category relationships of attributes. It infers more appropriate category relationships of intra-attribute values to suit current partition results, effectively improving the clusters through distance-cluster joint optimization.

\item \textbf{Providing a rigorous metric with convergence guarantee}: A series of theoretical proofs are provided to rigorously demonstrate that the joint learning process of subspace category relationship exhibits strong compatibility with Euclidean distance, along with high efficiency and convergence.
\end{itemize} 

\section{Related Work}
\label{sct:relate_work}

\textbf{Distance Measures} for categorical data are generally categorized into encoding-based and direct measurement approaches. Early methods, such as one-hot encoding with Hamming distance~\cite{hdm}, fail to capture the full spectrum of dissimilarity between intra-attribute values. Subsequent methods~\cite{adm, abdm, cde} refine distance discrimination but still focus on pairwise distances without considering overall category relationships. Information entropy-based approaches~\cite{lsm,ebdmjournal} improve this by better capturing intra-attribute relationships, while more recent methods~\cite{MCDC, udm, het2hom, zhang2024order} attempt to place attribute values along a distance axis, offering a linear representation similar to numerical attributes. However, these methods remain task-independent and may not adapt well to different clustering tasks.

\textbf{Distance Learning} combines metric optimization with clustering through iterative learning. Early work~\cite{ocil} focus on modeling sample-cluster similarities using probability distributions within clusters. Later methods, such as kernel-based techniques~\cite{untie} and joint embedding of numerical and categorical attributes~\cite{mai, mix2vec, QGRL}, further refine distance metrics. More recent advances~\cite{dlc, HDC, ADC, CoForest} unify metric learning with adaptive graph weighting, optimizing both cluster structures and distance functions simultaneously.

\textbf{Subspace Learning} focuses on identifying attribute subsets that enhance clustering. Early methods~\cite{wchd} fix attribute weights, limiting adaptability, while subsequent approaches~\cite{wksc} iterate on weight and cluster assignments. Methods like~\cite{WOCIL} integrate attribute weighting for mixed data clustering, and extensions~\cite{HARR} unify categorical and numerical attributes via geometric projections, aligning latent dimensions across attribute types.

These advances improve clustering performance, but the learned category relationships remain uniform across the global cluster distribution, neglecting subspace-specific variations. This limits their capability in representing complex relationships between categorical attribute values.

\section{Proposed Method}

In this section,  we first formulate the problem (\textbf{\S~\ref{sct:pf}}). Then the subspace category relationship (\textbf{\S~\ref{sct:Method_sdsc}}) is then inferred from the modeled value-level relation structure, and the corresponding distance metric (\textbf{\S~\ref{sct:Method_hicd}}) is defined. After that, a joint learning (\textbf{\S~\ref{sct:Method_joint}}) scheme is introduced to make the distance learnable with clustering.

\subsection{Problem Formulation}
\label{sct:pf}

The problem of categorical data clustering with subspace category relationship learning is formulated as follows. 
Given a categorical dataset $X=\{\mathbf{x}_1,\mathbf{x}_2,\ldots,\mathbf{x}_n\}$ with $n$ data samples. Each data sample can be denoted as an $l$-dimensional row vector $\mathbf{x}_i=[x_{i,1},x_{i,2},\ldots,x_{i,l}]^\top$ represented by $l$ attributes $A=\{\mathbf{a}_1,\mathbf{a}_2,...,\mathbf{a}_l\}$. Each attribute $\mathbf{a}_r$ can be denoted as a column vector $\mathbf{a}_r=[x_{1,r},x_{2,r},...,x_{n,r}]$ composed of the $r$-th values of all $n$ samples. The $n$ values can be viewed as sampled from a limited number of possible values $V_r=\{v_r^1,v_r^2,...,v_r^{o_r}\}$ with $o_r$ indicating the number of possible values of $\mathbf{a}_r$. The goal of our cluster-customized category relationships clustering can be formalized as minimizing the intra-cluster dissimilarity:
\begin{equation}\label{eq:obj_pre}
	z(\mathbf{H}, \mathbf{M}, \mathbf{G})=\sum_{j=1}^k\sum_{i=1}^nh_{i,j}\cdot\Phi(\mathbf{x}_i,\mathbf{m}_j),
\end{equation}
where $\mathbf{H}$, $\mathbf{M}$ and $\mathbf{G}$ are a $n \times k$ cluster partition matrix, a $k \times l$ centroid matrix, and a matrix containing $k \times l$ cluster-customized category relationships (also called subspace category relationships) represented as graphs, respectively. Specifically, cluster partition $\mathbf{H}$ is an affiliation matrix with its $(i,j)$-th entry $h_{i,j}$ indicating the affiliation between sample $\mathbf{x}_i$ and cluster $C_j$, which can be written as: 
\begin{equation}
	\label{eq:qim_pre}
	h_{i,j}=\left\{
	\begin{array}{ll}
		1,  & \text{if}\ j=\arg\min\limits_y\Phi(\mathbf{x}_i,\mathbf{m}_y)\\
		0,  & \text{otherwise}\\
	\end{array}\right.
\end{equation}
with $\sum_{j=1}^kh_{i,j}=1$. $\Phi(\mathbf{x}_i, \mathbf{m}_y)$ is the attribute-level dissimilarity between sample $\mathbf{x}_i$ and cluster $C_y$ described by the cluster center $\mathbf{m}_y$. The value of $\mathbf{m}_y$ is computed following the way of the conventional $k$-modes clustering algorithm, i.e., the value of $m_{y,r}$ is equal to the mode from $\mathbf{a}_r$ in $C_y$. The attribute-level dissimilarity is computed in the context of subspace category relationships can be written as:
\begin{equation}
	\label{eq:dist_m}	\Phi(\mathbf{x}_i,\mathbf{m}_y)=\sum_{r=1}^{l}\phi(x_{i,r},m_{y,r}; G_{y,r}),
\end{equation}
where $G_{y,r} \in \mathbf{G}$ is the category relationship of $r$-th attribute in $y$-th cluster, and $\phi(x_{i,r}, m_{y,r}; G_{y,r})$ is the corresponding sample-cluster dissimilarity. 

This paper focuses on defining the category relationship within each cluster to facilitate a more thorough optimization through clustering-customized distance metric learning. Therefore, the key is how to reasonably define the distance $\phi(x_{i,r}, m_{y,r}; G_{y,r})$ under an appropriate category relationship $G_{y,r}$ in the attribute aspect.

\subsection{Subspace Category Relationship Inference}
\label{sct:Method_sdsc}

To infer subspace category relationships for each cluster, we first model the value-level relation through a fully-connected graph as it can flexibly reflect any relationship among possible values. A graph $G_{j,r}$ indicating the category relationship of $r$-the attribute in $j$-th cluster is denoted as $ G_{j,r} = <V_r, B_r, \mathbf{W}_{j,r}>$, where $V_r$ includes $o_r$ possible values as nodes. $B_r$ is the set of $o_r(o_r-1)/2$ edges connecting all the $o_r$ nodes. $\mathbf{W}_{j,r}$ is the weight matrix reflecting edge lengths.

Given cluster partition $\mathbf{H}$, for each cluster, the weight between two nodes $v_r^{u}$ and $v_r^{s}$ within a cluster $C_j$ is defined as the difference between their Conditional Probability Distributions (CPDs) extracted from the current cluster by:
\begin{equation}
    \label{eq:compute_w}
    {{\mathbf{W}}_{j,r}(u,s)} = |{p}({v_r^{u}}| {C_j}) - p({v_r^{s}|{C_j}})|,
\end{equation}
where $\mathbf{W}_{j,r}(u,s)$ is the $(u,s)$-th element of weight matrix $\mathbf{W}_{j,r}$, and $p({{v_r^{u}}|{C_j}})$ is the conditional probability of $v_r^{u}$ given cluster $C_j$, which can be computed as:
\begin{equation}
    \label{eq:compute_p}
   p({v_r^{u}}|{C_j}) = \frac{{ card(\{{\mathbf{x}_i} \in {C_j}|{{\mathbf{x}}_{i,r}} = {v_r^{u}}\} })}{{card({C_j})}},
\end{equation}
where $card( \cdot )$ counts the number of samples in a set. A condition probability distribution $p({{v_r^{u}}|{C_j}})$ describes the distribution pattern of a value within $C_j$. Consequently, two values with similar patterns are considered to be more similar, and thus the weight of the edge connecting them is smaller. Although the weighted fully connected graph provides a flexible way to represent the relationships among the intra-attribute values in each cluster, it introduces uncertainties in the distance representation because the relationship between a pair of values can be described by multiple paths.

To refine the graph, and consider that the common definition of distance between two samples is based on the shortest path length~\cite{Graphthe,topo}, we extract the shortest paths that can connect all the possible values (i.e., all nodes in $G_{j,r}$) to form the Minimum Spanning Tree (MST) $\mathcal{T}_{j,r}$ (called relation tree hereinafter) by:
    \begin{equation}\label{eq:min_tree}
    \mathcal{T}_{j,r}=\arg\min\limits_{\mathcal{E}_t \in B_r}\sum_{b_r^{u,s}\in \mathcal{E}_{T}} \mathbf{W}_{j,r}(u,s),
    \end{equation}
where $\mathcal{E}_{t}$ is a subset of $B_r$. $\mathcal{E}_{t}$ contains $o_r-1$ edges connecting all nodes $V_r$. $b_r^{u,s}$ from $B_r$ is the weighted edge connect nodes $v_r^u$ and $v_r^s$. Consequently, for a cluster $C_j$, a series of relation trees $\mathcal{T}_j=\{\mathcal{T}_{j,1},\mathcal{T}_{j,2}...,\mathcal{T}_{j,r}\}$ is inferred corresponding to all the attributes, and the subspace category relationship of all clusters can be represented as a set $\mathcal{T}=\{\mathcal{T}_1, \mathcal{T}_2,..,\mathcal{T}_k\}$ corresponding to graph set $G$. 

\begin{theorem}
\label{thm:graph_tree_consistency}
The relation tree \( \mathcal{T}_{j,r} \) defines a deterministic and Euclidean-compatible distance metric over the categorical attribute values.
\end{theorem}
\begin{proof}
    It is proved that the relation tree  $\mathcal{T}_{j,r}$, inferred from the fully connected graph  $G_{j,r}$, retains all pairwise distances among nodes, establishing its determinism and uniqueness. Then, it is demonstrated that the distance metric induced by the relation tree always forms a linear structure, naturally compatible with the one-dimensional Euclidean geometry. A detailed proof can be found in \textbf{Appendix A.1}.
\end{proof}

\begin{remark} [Determinism and Euclidean Compatibility of the Inferred Relation Tree] \label{rmk:d&e_irt}
The inferred relation tree exhibits determinism and strong compatibility with Euclidean space (as proven in \textbf{Theorem~\ref{thm:graph_tree_consistency}}). The determinism of the relation tree ensures that the learned distance metric is stable and interpretable, providing clarity and eliminating the need for users to manually define or adjust category relationships. The compatibility with Euclidean space allows for easy extension to mixed datasets, making the method adaptable to both categorical and numerical attributes. This is further validated in \textbf{\S~\ref{sct: ablation}}, where our method achieves performance improvement and demonstrates superior results compared to other clustering methods.
\end{remark}

\subsection{Clustering-Customized Subspace Distance}
\label{sct:Method_hicd}

Given the subspace category relationship $\mathcal{T}$, we further tune them into customized distance metrics by modeling the distance between intra-attribute values to capture more compact clusters within subspaces. Specifically, for cluster $C_j$, given the relation tree $\mathcal{T}_{j,r}$ corresponding to attribute $\mathbf{a}_r$, the distance between values $v_r^{u}$ and $v_r^{s}$ can be computed by:
\begin{equation}
    \label{eq: tree-order_distance}
{\mathbf{D}_{j,r}(u,s)}{\rm{ = }}\sum_{\mathbf{W}_{j,r}(u,s)\in {T}} {\mathbf{W}_{j,r}(u,s)},
\end{equation}
where $\mathbf{D}$ is a $o_r \times o_r$ distance matrix of possible values, and $T$ is the set of edges splicing the shortest path from $v_r^{s}$ to $v_r^{u}$ in $\mathcal{T}_{j,r}$. Since this is computed based on the weights defined based on the intra-cluster statistics by \textbf{Eq.~(\ref{eq:compute_w})}, it is referred to as the clustering-customized subspace distance.

\begin{remark} [Generalized Attribute Weighting through Relation Tree] \label{rmk:gaw_rt}
The category relationship represented by the relation tree implicitly captures the importance of each attribute through a generalized attribute weighting mechanism. Specifically, the clustering-customized subspace distance model adapts to variations in intra-attribute distances, allowing for different magnitudes of distances between attribute values across different attributes. As indicated by Eqs~(\ref{eq:compute_w}) and~(\ref{eq: tree-order_distance}), a larger difference in attribute values leads to a more concentrated distribution of the $r$-th attribute within the cluster, making it more likely to be considered significant. This variability in the distribution of attribute values ensures that each attribute contributes differently when calculating attribute-level distances, thereby reflecting its importance within the overall category relationship. This approach can be viewed as a form of generalized attribute weighting, where the learned distances effectively decompose the attribute importance into a more detailed and dynamic layer of distance learning.
\end{remark}

Accordingly, the sample-cluster distance $\phi(x_{i,r}, m_{j,r})$ reflected by $\mathcal{T}_{j,r}$ for the $r$-th attribute can be defined as:
\begin{equation}
	\label{eq:dist_r}
        \phi ({{x}_{i,r}},{{m}_{j,r}};{\mathcal{T}_{j,r}}) = \mathbf{D}_{j,r}(u,s),
\end{equation}
where sample value $x_{i,r}$ and center value $m_{j,r}$ are assumed be equal to the possible values $v_r^u$ and $v_r^s$, respectively. The overall sample-cluster distance $\Phi(\mathbf{x}_i,\mathbf{m}_j)$ is defined as:
\begin{equation}
	\label{eq:dist_final}
	\Phi(\mathbf{x}_i,\mathbf{m}_j; \mathcal{T}_j)=\sum_{r=1}^{l}\phi(x_{i,r},m_{j,r};\mathcal{T}_{j,r}).
\end{equation}

\begin{theorem}\label{thm:sample_cluster_distance}
The sample-cluster distance $ \Phi (\mathbf{x}_i, \mathbf{m}_j; \mathcal{T}_j) $ defined in the context of $\mathcal{T}_j$ represents a valid distance metric.
\end{theorem}
\begin{proof}
It is proved that the distance measure $\mathbf{D}_{j,r}(u,s) $ defined within the context of the relation tree $\mathcal{T}_{j,r}$ constitutes a valid distance metric. Then, the sample-cluster distance $\Phi (\mathbf{x}_i, \mathbf{m}_j; \mathcal{T}_j)$ is shown to be a weighted sum of $\mathbf{D}_{j,r}(u,s)$ with non-negative weights, also satisfies the conditions of non-negativity, symmetry, and triangle inequality, thereby forming a valid distance metric. The detailed proof can be found in \textbf{Appendix A.2}.
\end{proof}

It can be seen that the cluster-customized distance satisfies the definition of a valid distance metric, which ensures consistent and reliable distance calculations across different clusters, thereby improving clustering accuracy.

\subsection{Joint Learning of Distance and Clusters}\label{sct:Method_joint}

To jointly optimize the distance metric and sample partition for better mutual adaptation, the subspace category relationship inference is integrated into the clustering objective, which can be written based on Eqs.~(\ref{eq:obj_pre}), (\ref{eq:dist_m}), and (\ref{eq:dist_r}):
\begin{equation}\label{eq:obj}
	z(\mathbf{H},M, \mathcal{T})=\sum_{j=1}^k\sum_{i=1}^nh_{i,j}\cdot\sum_{r=1}^{l}\mathbf{D}_{j,r}(u,s), 
\end{equation}
where we assume that $x_{i,r}=v_r^u$ and $m_{j,r}=v_r^s$. Then the problem can be formulated as computing $\mathbf{H}$, $M$, and $\mathcal{T}$ to minimize $z$, which can be summarized into the following three steps:
1) Fix $M$ and $\mathcal{T}$, compute $\mathbf{H}$; 2) Fix $\mathcal{T}$ and $\mathbf{H}$, compute $M$; 3) Repeat 1) and 2) until convergence, fix $\mathbf{H}$ and $M$ and reinfer $\mathcal{T}$ according to Eqs.~(\ref{eq:compute_w}),~(\ref{eq:compute_p}), and~(\ref{eq: tree-order_distance}). These three steps are repeated until the $\mathbf{H}$ no longer changes. The algorithm is summarized as \textbf{Algorithm 1} in Extended Version, which can be proven to guarantee convergence and is with complexity of $O(nlk{\mathcal{I}}{E})$, linear to $n$, $l$, and $k$. (Proof in \textbf{Appendix A.3} and \textbf{Appendix A.4}).

\begin{theorem} \label{the:converge}
DISC algorithm converges to a local minimum in a finite number of iterations.
\end{theorem}
\begin{proof} 
The inner loop is proved to attain a local minimum of $z(\mathbf{H}, M, \mathcal{T})$ within a finite number of iterations, when $\mathcal{T}$ is fixed. Then, the outer loop $E$ is proved to converges to a local minimum of the objective function $z(\mathbf{H}, M, \mathcal{T})$ in a finite number of iterations. The detailed proof can be found in \textbf{Appendix A.3}.
\end{proof}

\section{Experiment}

The experimental setup is first introduced in \textbf{\S~\ref{sct: setup}}. To comprehensively validate the effectiveness and superiority of our method, four experiments have been conducted. \textbf{\S~\ref{sct: cpAsig}}: Clustering performance evaluation; \textbf{\S~\ref{sct: ablation}}: Ablation studies; \textbf{\S~\ref{sct: cAe}}: Convergence and efficiency evaluation; \textbf{\S~\ref{sct: qa}}: Qualitative study of the clustering effect. The experiments are conducted on an Intel i7-13500H@2.60GHz workstation with 16G RAM and MATLAB R2022b.

\subsection{Experimental Setup}
\label{sct: setup}

\textbf{10 counterparts} are compared: KMD~\cite{kmd} is a traditional method, CBDM~\cite{cbdm_conf}, WOCIL~\cite{WOCIL}, CURE~\cite{cde}, CDCDR~\cite{CDCDR}, and HDC~\cite{HDC} are subspace clustering baselines, while CoForest~\cite{CoForest}, MCDC~\cite{MCDC}, SigDT~\cite{SigDT}, and HARR~\cite{HARR} are State-Of-The-Art (SOTA) distance metric learning and statistical-based methods. Each method is implemented 10 times and the average performance is reported.

\noindent \textbf{12 Datasets} from various domains are utilized for the experiments. All the datasets are real public datasets collected from the UCI Machine Learning Repository~\cite{uci}, and their statistical information is shown in \textbf{Table~\ref{tb:statistics}}. For all the counterparts, we set $k$ at $k^*$, i.e., the true number of clusters given by the dataset labels.

\begin{table}[!t]
\caption{Dataset statistics. $l$, $n$, and $k^*$ are the numbers of attributes, samples, and true number of clusters, respectively.}
\label{tb:statistics}
\centering
\resizebox{0.8\columnwidth}{!}{
\begin{tabular}{r|cc|ccc}
\toprule
No. & Dataset & Abbrev. & $l$ & $n$ \ \ \ & $k^*$ \\
\midrule
1& Car Evaluation & CA & 6 & 1728 & 4\\
2& Soybean (large) & SB &35 & 266 & 15 \\	
3& Nursery School & NU & 8 & 12960 & 4 \\
4& Contraceptive Choice & CC &7 & 1473 & 3 \\	
5& Amphibians & AP &12 & 189 & 2 \\	
6& Dermatology & DT &33 & 366 & 6 \\	
7& Auction Verifications & AV &6 & 2043 & 2 \\
8& Obesity Levels & OB & 8 & 2111 & 7 \\	
9& Teaching Assistant& TA & 4 & 151 & 3 \\	
10& Bank Marking & BM & 9 & 45211 & 2 \\	
11& Congressional Voting & CV & 16 & 435 & 2 \\	
12& Zoo & ZO &16 & 101 & 7 \\		
\bottomrule
\end{tabular}}
\end{table}

\noindent \textbf{Three validity indices}, including clustering accuracy (ACC)~\cite{ex1} and ARI coefficients~\cite{ex3,ex4}. To ensure a fair comparison, an entropy-based clustering CoMPactness (CMP) measure is introduced to eliminate the influence of different metric magnitudes, which is computed:
\begin{equation}
\label{eq:cmp}
\text{CMP} = \frac{1}{lk} \sum_{j=1}^{k} \sum_{r=1}^{l} \sum_{u=1}^{o_r}\frac{ -p(v_r^u | C_j)\log p(v_r^u | C_j)}{\log o_r}.
\end{equation}
The CMP quantifies attribute value concentration within clusters and approaches zero for highly concentrated attribute values, indicating cluster compactness.

\begin{table*}[!t]
\caption{Clustering performance evaluated by the external ACC and ARI, and internal CMP. Larger values of ACC (ranging from 0 to 1) and ARI (ranging from -1 to 1) and smaller values of CMP (ranging from 0 to 1) indicate better clustering performance. The best and second-best results on each dataset are highlighted in \textbf{bold} and \underline{underlined}, respectively.}
\label{tb:clustering}
\centering
\resizebox{2.05\columnwidth}{!}{
\begin{tabular}{l|l|ccccccccccc}
\toprule
 \multirow{2}{*}{Indices}&  \multirow{2}{*}{Data} &  {KMD } & {CBDM} &  {WOCIL}  &  {CURE}   &  {CDCDR}  &  {HDC}  &  {CoForest} &  {MCDC}  &  {SigDT}  &   {HARR}&  {DISC} \\

 & & [DMKD'98]& [TKDD'12]& [TNNLS'17] &[TKDE'18]&[PR'22] &[TPAMI'22] &[ECAI'24]&[ICDCS'24] &[In Sci'25] & [ESWA'25] & (ours) \\
\midrule

\multirow{13}{*}{ACC}&CA& 0.3795$\pm$0.04	& -	& 0.3665$\pm$0.06	& 0.4174$\pm$0.05	& 0.3376$\pm$0.06	& 0.3800$\pm$0.04	& 0.4024$\pm$0.05	& 0.3258$\pm$0.07	& \underline{0.5336$\pm$0.00}	& 0.3678$\pm$0.09	& \textbf{0.5826$\pm$0.06}\\	
&SB& 0.4970$\pm$0.05	& 0.5248$\pm$0.05	& 0.5778$\pm$0.03	& 0.5711$\pm$0.05	& 0.5805$\pm$0.06	& 0.5226$\pm$0.06	& \textbf{0.5996$\pm$0.02}	& 0.4846$\pm$0.04	& 0.5301$\pm$0.00	& 0.5421$\pm$0.04	& \underline{0.5865$\pm$0.06}\\	
&NU& 0.3315$\pm$0.03	& -	& 0.3106$\pm$0.12	& 0.3110$\pm$0.02	& 0.2977$\pm$0.00	& 0.3315$\pm$0.03	& 0.2974$\pm$0.06	& 0.2547$\pm$0.01	& \textbf{0.4020$\pm$0.00}	& 0.3127$\pm$0.04	& \underline{0.3908$\pm$0.05}\\	
&CC& 0.3800$\pm$0.02	& 0.4087$\pm$0.02	& 0.3854$\pm$0.02	& 0.3999$\pm$0.03	& 0.4144$\pm$0.02	& 0.4269$\pm$0.01	& 0.4016$\pm$0.02	& \underline{0.4304$\pm$0.01}	& 0.3374$\pm$0.00	& 0.4284$\pm$0.01	& \textbf{0.4432$\pm$0.01}\\	
&AP& 0.5354$\pm$0.03	& 0.5619$\pm$0.04	& 0.6095$\pm$0.01	& \underline{0.6164$\pm$0.01}	& 0.6026$\pm$0.01	& 0.5667$\pm$0.03	& 0.5497$\pm$0.03	& 0.5545$\pm$0.03	& 0.5132$\pm$0.00	& 0.5630$\pm$0.02	& \textbf{0.6413$\pm$0.04}\\	
&DT& 0.5972$\pm$0.12	& 0.6201$\pm$0.15	& 0.6000$\pm$0.11	& 0.7433$\pm$0.15	& 0.7050$\pm$0.11	& 0.6648$\pm$0.14	& 0.6757$\pm$0.09	& 0.5712$\pm$0.11	& \underline{0.8436$\pm$0.00}	& 0.7109$\pm$0.11	& \textbf{0.8592$\pm$0.06}\\	
&AV& 0.6251$\pm$0.08	& 0.6695$\pm$0.10	& 0.6610$\pm$0.13	& 0.6564$\pm$0.09	& \underline{0.7068$\pm$0.10}	& 0.6364$\pm$0.10	& 0.6669$\pm$0.10	& 0.6348$\pm$0.12	& 0.6383$\pm$0.00	& 0.6441$\pm$0.14	& \textbf{0.7258$\pm$0.11}\\	
&OB& 0.3403$\pm$0.05	& 0.3590$\pm$0.04	& 0.3622$\pm$0.03	& 0.3794$\pm$0.03	& 0.2923$\pm$0.04	& 0.3651$\pm$0.04	& \underline{0.3828$\pm$0.02}	& 0.3596$\pm$0.03	& 0.3813$\pm$0.00	& 0.3732$\pm$0.01	& \textbf{0.4251$\pm$0.02}\\	
&TA& 0.4073$\pm$0.03	& 0.4272$\pm$0.02	& 0.4245$\pm$0.01	& 0.4225$\pm$0.03	& 0.3907$\pm$0.03	& 0.4159$\pm$0.02	& 0.4132$\pm$0.03	& \underline{0.4298$\pm$0.03}	& 0.4172$\pm$0.00	& 0.4252$\pm$0.01	& \textbf{0.4391$\pm$0.05}\\	
&BM& 0.6395$\pm$0.10	& 0.6403$\pm$0.10	& 0.5997$\pm$0.07	& 0.5804$\pm$0.06	& \underline{0.6977$\pm$0.11}	& 0.5569$\pm$0.03	& 0.6087$\pm$0.07	& 0.6052$\pm$0.08	& 0.2492$\pm$0.00	& 0.6168$\pm$0.09	& \textbf{0.7290$\pm$0.10}\\	
&CV& 0.8628$\pm$0.01	& 0.8754$\pm$0.00	& 0.8671$\pm$0.00	& 0.8455$\pm$0.08	& 0.8462$\pm$0.08	& 0.8736$\pm$0.00	& \textbf{0.8768$\pm$0.00}	& 0.8414$\pm$0.08	& 0.5816$\pm$0.00	& 0.8736$\pm$0.00	& \underline{0.8759$\pm$0.00}\\	
&ZO& 0.6970$\pm$0.09	& 0.7287$\pm$0.07	& 0.6881$\pm$0.08	& \underline{0.7693$\pm$0.07}	& 0.7574$\pm$0.10	& 0.7287$\pm$0.08	& 0.7238$\pm$0.13	& 0.6257$\pm$0.07	& 0.7228$\pm$0.00	& 0.7257$\pm$0.03	& \textbf{0.8050$\pm$0.08}\\


\midrule
\multicolumn{2}{c|}{Average Rank} &8.38   & 6.38   & 6.75   & 5.25  &    5.92 &   6.29  &  5.42   & 8.33  &  6.83 &  5.21 & 1.25\\	
\bottomrule

\multirow{13}{*}{ARI}&CA& 0.0341$\pm$0.04	& -	& 0.0433$\pm$0.05	& 0.0548$\pm$0.05	& 0.0027$\pm$0.01	& 0.0347$\pm$0.03	& \underline{0.0668$\pm$0.05}	& 0.0043$\pm$0.01	& -0.0678$\pm$0.00	& 0.0433$\pm$0.05	& \textbf{0.0964$\pm$0.07}\\	
&SB& 0.3245$\pm$0.04	& 0.3540$\pm$0.03	& 0.3755$\pm$0.02	& 0.4071$\pm$0.05	& \underline{0.4104$\pm$0.07}	& 0.3502$\pm$0.04	& 0.4057$\pm$0.02	& 0.2941$\pm$0.03	& 0.3811$\pm$0.00	& 0.3868$\pm$0.03	& \textbf{0.4135$\pm$0.04}\\	
&NU& 0.0539$\pm$0.01	& -	& 0.0679$\pm$0.15	& 0.0350$\pm$0.02	& \underline{0.0730$\pm$0.00}	& 0.0539$\pm$0.01	& 0.0451$\pm$0.07	& 0.0027$\pm$0.00	& 0.0310$\pm$0.00	& 0.0327$\pm$0.05	& \textbf{0.0761$\pm$0.07}\\	
&CC& 0.0043$\pm$0.01	& 0.0289$\pm$0.01	& 0.0072$\pm$0.01	& 0.0164$\pm$0.02	& 0.0218$\pm$0.01	& 0.0313$\pm$0.01	& 0.0211$\pm$0.01	& 0.0311$\pm$0.01	& 0.0231$\pm$0.00	& \underline{0.0330$\pm$0.00}	& \textbf{0.0419$\pm$0.01}\\	
&AP& 0.0002$\pm$0.01	& 0.0167$\pm$0.02	& 0.0410$\pm$0.01	& \underline{0.0481$\pm$0.01}	& 0.0342$\pm$0.01	& 0.0151$\pm$0.02	& 0.0055$\pm$0.02	& -0.0007$\pm$0.02	& 0.0470$\pm$0.00	& 0.0022$\pm$0.01	& \textbf{0.0783$\pm$0.04}\\	
&DT& 0.4342$\pm$0.15	& 0.5450$\pm$0.21	& 0.4774$\pm$0.16	& 0.6905$\pm$0.18	& 0.6752$\pm$0.13	& 0.5813$\pm$0.19	& 0.6163$\pm$0.11	& 0.4296$\pm$0.16	& \underline{0.7821$\pm$0.00}	& 0.6520$\pm$0.11	& \textbf{0.8122$\pm$0.07}\\	
&AV& 0.0073$\pm$0.02	& 0.0365$\pm$0.02	& 0.0427$\pm$0.09	& 0.0319$\pm$0.03	& \underline{0.0578$\pm$0.07}	& 0.0199$\pm$0.02	& 0.0200$\pm$0.04	& 0.0094$\pm$0.02	& \textbf{0.0752$\pm$0.00}	& 0.0214$\pm$0.03	& 0.0390$\pm$0.05\\	
&OB& 0.1286$\pm$0.06	& 0.1389$\pm$0.04	& 0.1374$\pm$0.03	& 0.1669$\pm$0.05	& 0.0958$\pm$0.03	& 0.1497$\pm$0.05	& 0.1652$\pm$0.03	& 0.1435$\pm$0.03	& 0.1674$\pm$0.00	& \underline{0.2032$\pm$0.02}	& \textbf{0.2255$\pm$0.01}\\	
&TA& 0.0144$\pm$0.01	& 0.0285$\pm$0.01	& 0.0310$\pm$0.02	& 0.0188$\pm$0.01	& 0.0120$\pm$0.01	& 0.0166$\pm$0.01	& 0.0264$\pm$0.02	& \underline{0.0377$\pm$0.02}	& 0.0197$\pm$0.00	& 0.0113$\pm$0.00	& \textbf{0.0394$\pm$0.03}\\	
&BM& 0.0281$\pm$0.02	& 0.0094$\pm$0.06	& 0.0118$\pm$0.02	& -0.0163$\pm$0.03	& \textbf{0.0693$\pm$0.11}	& 0.0012$\pm$0.02	& 0.0217$\pm$0.01	& 0.0087$\pm$0.03	& -0.0177$\pm$0.00	& 0.0209$\pm$0.01	& \underline{0.0472$\pm$0.05}\\	
&CV& 0.5254$\pm$0.02	& 0.5627$\pm$0.00	& 0.5381$\pm$0.00	& 0.4948$\pm$0.17	& 0.4989$\pm$0.18	& 0.5572$\pm$0.00	& \textbf{0.5668$\pm$0.00}	& 0.4831$\pm$0.17	& 0.3651$\pm$0.00	& 0.5572$\pm$0.00	& \underline{0.5633$\pm$0.00}\\	
&ZO& 0.6456$\pm$0.12	& 0.6822$\pm$0.10	& 0.6353$\pm$0.12	& \underline{0.7350$\pm$0.11}	& 0.7072$\pm$0.12	& 0.6847$\pm$0.12	& 0.6912$\pm$0.17	& 0.4950$\pm$0.16	& 0.5920$\pm$0.00	& 0.6765$\pm$0.05	& \textbf{0.7736$\pm$0.11}\\

\midrule
\multicolumn{2}{c|}{Average Rank} &8.38   & 6.75   & 6.13   & 5.42   & 5.33 &   6.50    &5.08  &  8.58   &6.58   &5.83 &1.42\\	
\bottomrule

\multirow{13}{*}{CMP}&CA& 0.8852$\pm$0.00	& -	& 0.7960$\pm$0.02	& 0.8026$\pm$0.02	& 0.8271$\pm$0.01	& 0.8852$\pm$0.00	& \underline{0.7932$\pm$0.02}	& 0.8352$\pm$0.02	& 0.8792$\pm$0.00	& 0.8100$\pm$0.04	& \textbf{0.6557$\pm$0.06}\\	
&SB& 0.3101$\pm$0.02	& 0.2936$\pm$0.02	& \textbf{0.2641$\pm$0.01}	& 0.2799$\pm$0.01	& 0.2824$\pm$0.01	& 0.2862$\pm$0.01	& 0.2738$\pm$0.01	& 0.3560$\pm$0.02	& 0.3899$\pm$0.00	& 0.2949$\pm$0.01	& \underline{0.2714$\pm$0.01}\\	
&NU& 0.8908$\pm$0.00	& -	& 0.7948$\pm$0.04	& 0.8162$\pm$0.02	& 0.8750$\pm$0.00	& 0.8908$\pm$0.00	& \underline{0.7688$\pm$0.01}	& 0.8416$\pm$0.02	& 0.9144$\pm$0.00	& 0.7735$\pm$0.02	& \textbf{0.6610$\pm$0.06}\\	
&CC& 0.6270$\pm$0.03	& 0.5626$\pm$0.03	& 0.5988$\pm$0.02	& 0.5942$\pm$0.02	& 0.5828$\pm$0.02	& 0.5892$\pm$0.02	& 0.5603$\pm$0.02	& 0.5799$\pm$0.02	& \underline{0.5085$\pm$0.00}	& 0.5740$\pm$0.01	& \textbf{0.4222$\pm$0.04}\\	
&AP& 0.6083$\pm$0.01	& 0.6044$\pm$0.02	& 0.5738$\pm$0.01	& 0.5814$\pm$0.00	& 0.5730$\pm$0.00	& 0.5850$\pm$0.02	& 0.5771$\pm$0.02	& \underline{0.5378$\pm$0.05}	& \textbf{0.5285$\pm$0.00}	& 0.6211$\pm$0.05	& 0.5692$\pm$0.03\\	
&DT& 0.4240$\pm$0.02	& 0.4065$\pm$0.03	& 0.3942$\pm$0.01	& 0.3886$\pm$0.02	& 0.3804$\pm$0.02	& 0.3830$\pm$0.03	& 0.3954$\pm$0.01	& 0.4250$\pm$0.03	& \underline{0.3673$\pm$0.00}	& 0.3947$\pm$0.01	& \textbf{0.3314$\pm$0.02}\\	
&AV& 0.7039$\pm$0.02	& 0.6623$\pm$0.05	& 0.6246$\pm$0.06	& 0.6194$\pm$0.08	& 0.6798$\pm$0.05	& 0.7011$\pm$0.02	& 0.6260$\pm$0.03	& 0.6317$\pm$0.06	& \textbf{0.4809$\pm$0.00}	& 0.7073$\pm$0.05	& \underline{0.6042$\pm$0.10}\\	
&OB& 0.3324$\pm$0.04	& 0.3617$\pm$0.05	& 0.3128$\pm$0.03	& 0.3289$\pm$0.04	& 0.4115$\pm$0.03	& 0.3646$\pm$0.05	& 0.2870$\pm$0.03	& 0.2779$\pm$0.03	& 0.3824$\pm$0.00	& \underline{0.2445$\pm$0.02}	& \textbf{0.1793$\pm$0.05}\\	
&TA& 0.4537$\pm$0.05	& 0.4393$\pm$0.06	& \underline{0.3701$\pm$0.02}	& 0.4993$\pm$0.04	& \textbf{0.3273$\pm$0.07}	& 0.5097$\pm$0.04	& 0.3941$\pm$0.04	& 0.4701$\pm$0.08	& 0.6161$\pm$0.00	& 0.5332$\pm$0.01	& \underline{0.3701$\pm$0.06}\\	
&BM& 0.6293$\pm$0.02	& 0.6302$\pm$0.04	& 0.6016$\pm$0.03	& 0.6170$\pm$0.03	& 0.6264$\pm$0.04	& 0.6265$\pm$0.01	& 0.5806$\pm$0.01	& 0.6067$\pm$0.01	& \textbf{0.3462$\pm$0.00}	& 0.5883$\pm$0.04	& \underline{0.5530$\pm$0.06}\\	
&CV& 0.5514$\pm$0.00	& 0.5465$\pm$0.00	& 0.5457$\pm$0.00	& 0.5643$\pm$0.06	& 0.5616$\pm$0.06	& 0.5450$\pm$0.00	& 0.5435$\pm$0.00	& 0.5669$\pm$0.07	& \textbf{0.5099$\pm$0.00}	& 0.5454$\pm$0.00	& \underline{0.5328$\pm$0.00}\\	
&ZO& 0.2582$\pm$0.01	& 0.2615$\pm$0.02	& 0.2616$\pm$0.01	& \underline{0.2550$\pm$0.02}	& 0.2640$\pm$0.02	& 0.2714$\pm$0.01	& 0.2566$\pm$0.01	& 0.3288$\pm$0.05	& 0.3442$\pm$0.00	& 0.2563$\pm$0.00	& \textbf{0.2498$\pm$0.01}\\

\midrule
\multicolumn{2}{c|}{Average Rank} &8.58&	7.92&	4.83&	5.83&	6.42	&7.67	&3.92	&7.08&	5.75&6.17&	1.58
\\	
\bottomrule

\end{tabular}}
\end{table*}

\begin{table*}[!t]
\caption{ACC performance of ablated DISC variants formed by successively ablating the two core components: RL, CPDs. Red arrows indicate performance degradation over the former more complete variant. }
\label{tb:ab_acc}
\centering
\resizebox{2.06\columnwidth}{!}{
\begin{tabular}{cc|l|cccccccccccc|c}
\toprule
\multicolumn{2}{c|}{DISC Components} & \multirow{2}{*}{Variants} & \multicolumn{12}{c|}{Datasets} & \multirow{2}{*}{Average} \\
\cmidrule(l{-0.15em}r){1-2} \cmidrule(lr){4-15} 
RL & CPDs &  & CA & SB & NU & CC & AP & DT & AV & OB & TA & BM & CV & ZO & Rank \\
\midrule
$\checkmark$ & $\checkmark$ & DISC & 
0.5826 &  
0.5759 &  
0.3918 &  
0.4432 &  
0.6413 &  
0.8592 &  
0.7258 &  
0.4251 &  
0.4391 &  
0.7290 &  
0.8759 &  
0.8050 &  
1.08 \\

& $\checkmark$ &${I}$ & 
~~0.5524\textcolor{darkred}{$\downarrow$} &  
~~0.5568\textcolor{darkred}{$\downarrow$} &  
~~0.3876\textcolor{darkred}{$\downarrow$} &  
~~0.4193\textcolor{darkred}{$\downarrow$} &  
0.6423&  
~~0.7003\textcolor{darkred}{$\downarrow$} &  
~~0.7018\textcolor{darkred}{$\downarrow$} &  
~~0.3963\textcolor{darkred}{$\downarrow$} &  
~~0.4377\textcolor{darkred}{$\downarrow$} &  
~~0.6893\textcolor{darkred}{$\downarrow$} &  
~~0.8722\textcolor{darkred}{$\downarrow$} &  
~~~0.8050\textcolor{gray}{$-$} &  
1.83 \\

& &${II}$ & 
~~0.3795\textcolor{darkred}{$\downarrow$} & ~~0.4959\textcolor{darkred}{$\downarrow$} & ~~0.3315\textcolor{darkred}{$\downarrow$} & ~~0.3800\textcolor{darkred}{$\downarrow$} & ~~0.5354\textcolor{darkred}{$\downarrow$} & ~~0.5972\textcolor{darkred}{$\downarrow$} & ~~0.6251\textcolor{darkred}{$\downarrow$} & ~~0.3478\textcolor{darkred}{$\downarrow$} & ~~0.4073\textcolor{darkred}{$\downarrow$} & ~~0.6395\textcolor{darkred}{$\downarrow$} & ~~0.8628\textcolor{darkred}{$\downarrow$} & ~~0.6970\textcolor{darkred}{$\downarrow$} & 3.00 \\
\bottomrule
\end{tabular}}
\end{table*}

\subsection{Clustering Performance Evaluation}
\label{sct: cpAsig}

The clustering performance of various methods is evaluated using both external (ACC and ARI) and internal (CMP) validation indices, as shown in \textbf{Table~\ref{tb:clustering}}. External validation measures how well the clustering matches the true labels, while internal validation assesses the compactness and distinguishability of the clusters.
 
\textbf{External (label-based) validation}: Overall, DISC outperforms most methods, as shown in Table~\ref{tb:clustering}. The observations include: 1) DISC consistently performs well, demonstrating its clustering superiority across all datasets. 2) On the ACC index, DISC is slightly behind the best on SB, NU, and CV datasets (differences under 0.015), still proves the competitiveness of DISC. 3) DISC performs similarly to the best methods on the ARI index, with small differences on AV, BM, and CV datasets. 4) SigDT performs well in ACC but poorly in ARI on CA and NU datasets. The reason is that it over-aggregates data into a single large cluster, inflating ACC while failing to capture inter-cluster differences, thus lowering ARI. In contrast, DISC maintains a higher ARI, demonstrating its ability to preserve true cluster structures.

\textbf{Internal validation}: As can be seen in \textbf{Table~\ref{tb:clustering}}, DISC achieves the best rank in the CMP index, indicating superior cluster compactness and distinguishability. It can be observed that DISC performs worse than SigDT in some cases (e.g., on AP, AV, BM and CV datasets). This is because on these datasets, SigDT tends to partition samples into excessively small clusters, which may surely demonstrate higher compactness but is relatively incompetent in exploring a proper number of prominent clusters, as confirmed in its weaker external validation results. It is also worth noting that the entropy-based evaluation index we use still has some limitations. That is, it actually measures the value inconsistency within different clusters by only distinguishing whether the values are identical or different. So it is relatively incompetent in precisely quantifying the effectiveness of the adopted distance measure in grouping similar samples to form compact clusters. In addition, the performance of CBDM is not reported for the NS and CE datasets, as the attributes are independent, making CBDM ineffective.

\subsection{Ablation Studies}
\label{sct: ablation}

To demonstrate the effectiveness of DISC's core components and verify the Euclidean compatibility of our distance metric, several ablated variants are compared from both algorithmic and data perspectives, as shown in Tables~\ref{tb:ab_acc} and~\ref{tb:mix_ablation}.

\textbf{From the algorithmic perspective}, we evaluate the proposed subspace category relationship learning mechanism by modifying DISC to DISC$^{I}$, which infers the category Relationship once without iterative Learning (RL). Additionally, DISC$^{II}$ adopts the traditional Hamming distance and uses the subspace Condition Probability Distribution-based measure (CPDs) for weight computation. \textbf{Table~\ref{tb:ab_acc}} compares the clustering accuracy of these variants. It can be seen that DISC outperforms the others on 11 out of 12 datasets, confirming the effectiveness of the proposed metric learning mechanism. DISC$^{I}$ with CPD-based weights consistently outperforms DISC$^{II}$ with Hamming distance, demonstrating the superiority of the probability distribution in distinguishing dissimilarity and improving clustering accuracy.

\begin{table}[!t]
\centering
\caption{Clustering Accuracy on mixed datasets, formed by adding the omitted numerical attributes back to the categorical datasets in \textbf{Table~\ref{tb:statistics}}. \textbf{Bold} indicates the best result. }

\label{tb:mix_ablation}
\resizebox{0.9\columnwidth}{!}{
\begin{tabular}{r|ccc|cc}
\toprule
Data & KPT & WOCIL & HARR  & DISC & DISC (mixed) \  \\
\midrule
CC	&0.3983&	0.3858	&0.4303&0.4432&\textbf{0.4494} \\
AP	&0.5158	&0.5723	&0.5757& 0.6413&\textbf{0.6556}	\\
DT &0.5313	&0.6901	&0.6145 & 0.8592&\textbf{0.8615} \\
AV	&0.539	&0.6733	&0.6805& 0.7258&\textbf{0.7764}	 \\
BM	&0.5865	&0.5724	&0.5222 & 0.7290&\textbf{0.7423}  \\

\bottomrule
\end{tabular}}
\end{table}

\textbf{From the data perspective}, to validate that our metric is compatible with the Euclidean distance, we combine the learned subspace distance metric for categorical attributes with the Euclidean distance for k-prototypes clustering, resulting in the DISC (mixed) method. \textbf{Table~\ref{tb:mix_ablation}} compares it with k-prototypes (KPT)~\cite{kpt} and two mixed data clustering methods, WOCIL~\cite{WOCIL} and HARR~\cite{HARR}. DISC (mixed) consistently outperforms all competing methods across the five mixed datasets. Compared to the original DISC, which processes only categorical attributes, the mixed datasets extension achieves further performance gains, demonstrating seamless integration with Euclidean distance and yielding excellent results without complex operations. This simplicity validates its compatibility with Euclidean space, as discussed in \textbf{Theorem~\ref{thm:graph_tree_consistency}}, enabling DISC to handle heterogeneous data types effectively and making it a promising solution for real-world clustering.

\begin{figure}[!t] 
    \centerline{\includegraphics[width=1.0\linewidth]{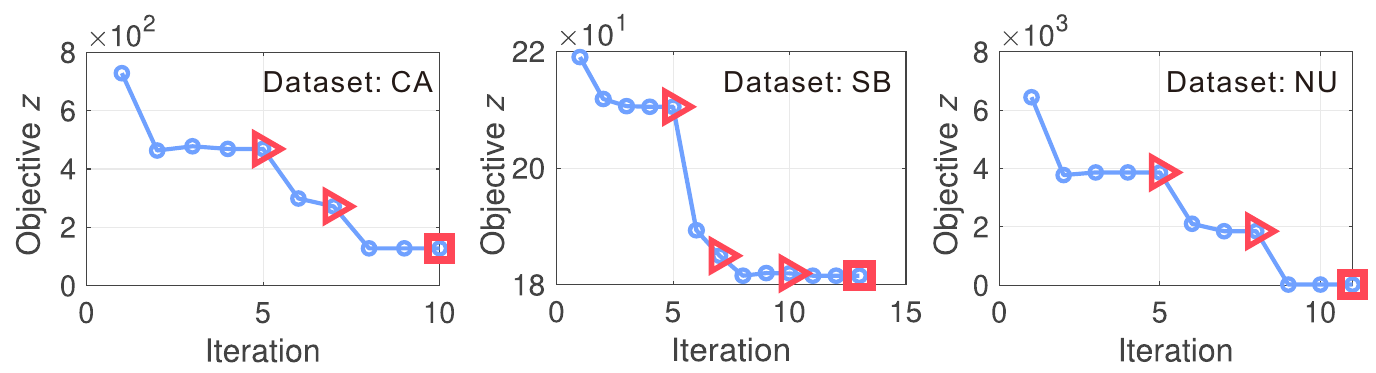}}
    \caption{Convergence curves of DISC on different datasets.} 
    \label{fig:Converge12}
\end{figure}

\begin{figure}[!t]	
 \centerline{\includegraphics[width=1.0\linewidth]{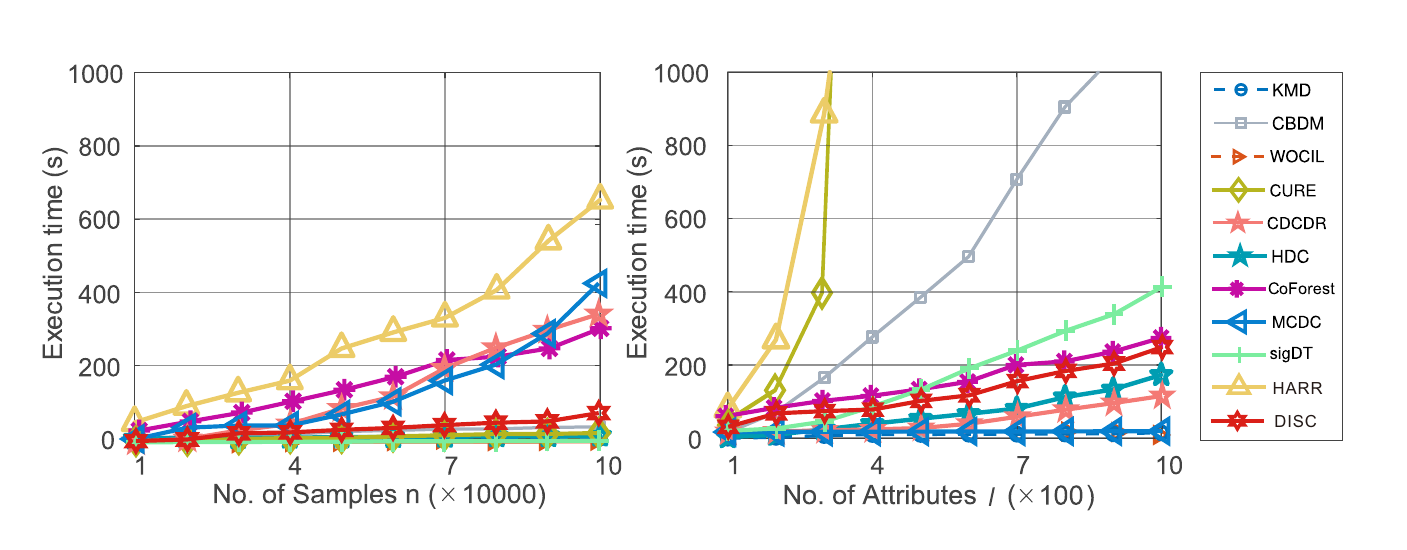}}
\caption{Execution time (y-axis) on synthetic datasets with different numbers of samples $n$ and attributes $l$ (x-axis).}	
\label{fig:time_complex}	
\end{figure}

\subsection{Converge and Efficiency Analysis}\label{sct: cAe}

\textbf{To assess the convergence} of DISC, we execute it on CA, SB, and NU datasets and present the convergence curves in \textbf{Figure~\ref{fig:Converge12}}. The horizontal axis represents the number of iterations, while the vertical axis shows the objective function value $z$. The red triangles and squares correspond to updates of the relation trees and the iterations where DISC converges, respectively. After each relation tree update, $z$ consistently decreases, indicating alignment with the optimization objective. DISC converges within 20 iterations, with a maximum of 10 updates, as expected in \textbf{Theorem~\ref{the:converge}}.

\textbf{To evaluate the efficiency of DISC}, large synthetic datasets are generated with varying attributes and sample sizes. Two experiments are conducted: 1) Fixing the number of attributes at $l = 20$ and varying the sample size $n$ from 10,000 to 100,000; 2) Fixing the sample size at $n = 2,000$ and varying the number of attributes $l$ from 100 to 1000. Each attribute has five values, and $k = 5$. Execution times for all 11 methods are shown in \textbf{Figure~\ref{fig:time_complex}}. DISC outperforms HARR, MCDC, CDCDR, and CoForest on large datasets and HARR, CURE, CBDM, sigDT, and CoForest on high-dimensional datasets. The execution time of DISC increases linearly with the number of samples or attributes, consistent with the time complexity analysis. In summary, DISC is efficient compared to State-Of-Art (SOTA) methods and incurs low computational cost.

\begin{figure}[!t]
    \centerline{\includegraphics[width=3.1in]{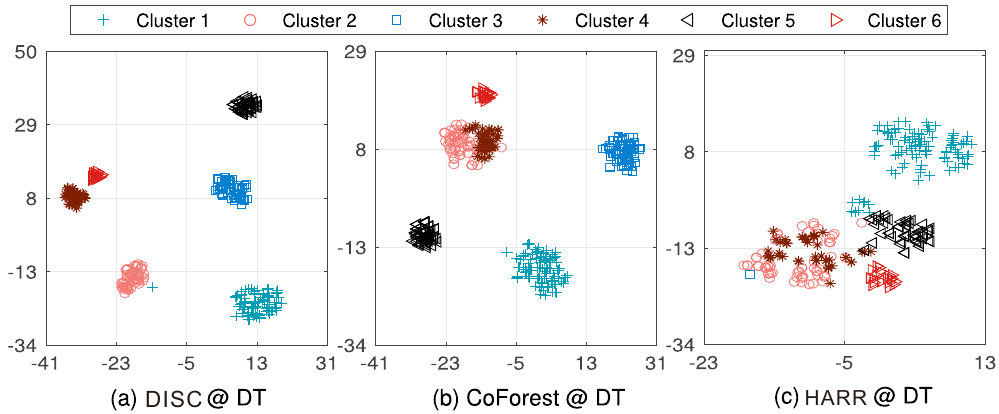}}
    \caption{t-SNE visualization of the DT dataset.}
    \label{fig:Instudy}
\end{figure}

\subsection{Qualitative Results}\label{sct: qa}

\textbf{To demonstrate the cluster discrimination ability} of the DISC method, the attribute distances learned by DISC, CoForest, and HARR are utilized to encode the attributes of the DT dataset for 2-dimensional t-SNE~\cite{r1visual} visualization. It can be seen from \textbf{Figure~\ref{fig:Instudy}} that DISC has significantly better cluster discrimination ability, as it performs clustering task-oriented distance learning to better suit the exploration of the $k$ clusters of the DT dataset.

\section{Concluding Remarks}
\label{sct:Conclude}

This paper addresses a key challenge in categorical data clustering, where existing distance learning methods are limited by uniform category relationships across different attributes and clusters. A novel learning framework named DISC is introduced to integrate subspace category relationships into clustering and jointly optimize them. During the learning process, DISC infers more appropriate subspace category relationships based on current partition, leading to more compact clusters and superior accuracy. Theoretical analysis shows that the designed optimization algorithm of DISC converges quickly with guarantee, and also proves that the learned category relationships are both deterministic and Euclidean-compatible. Extensive experiments on 12 benchmark datasets with statistical evidence show that the proposed DISC significantly outperforms various state-of-the-art methods. By bridging the gap in distance metric formulation between categorical data and clustering analysis, DISC lays a solid foundation for addressing critical subsequent challenges in categorical and mixed data clustering, including the handling of high-dimensional and noisy data as well as the automatic determination of the optimal number of clusters, which have yet to be fully explored both in the literature and in our present study.

\section{Acknowledgments}
This work was supported in part by the National Natural Science Foundation of China (NSFC) under grants: 62476063, 62376233 and 62306181, the NSFC/Research Grants Council (RGC) Joint Research Scheme under the grant N\_HKBU214/21, the Natural Science Foundation of Guangdong Province under grant: 2025A1515011293, the Natural Science Foundation of Fujian Province under grant: 2024J09001, the National Key Laboratory of Radar Signal Processing under grant: JKW202403, the General Research Fund of RGC under grants: 12201321, 12202622, and 12201323, the RGC Senior Research Fellow Scheme under grant: SRFS2324-2S02, Guangdong and Hong Kong Universities ``1+1+1'' Joint Research Collaboration Scheme with the grant: 2025A0505000004, the grant for Faculty Niche Research Areas of Hong Kong Baptist University with the code: RC-FNRA-IG/23-24/SCI/02, the Guangdong Higher Education Upgrading Plan (2021–2025), the Shenzhen Science and Technology Program under grant: RCBS20231211090659101 and the Xiaomi Young Talents Program.
\bibliography{aaai2026}

\appendix

\section{Lemmas, Theromes and Algorithm}

\subsection{Deterministic and Euclidean-compatible Distance Metric}

\begin{lemma}
    \label{lm:tree}
    The relation tree $\mathcal{T}_{j,r}$ inferred from the fully connected graph $G_{j,r}$ retains all pairwise distances among nodes is deterministic and unique.
\end{lemma}
\begin{proof}
    Given a fully connected graph $G_{j,r}$, for the convenience of analysis without loss of generality, we first reindex the nodes such that their conditional probabilities satisfy: 
    \begin{equation}
         p(v_r^1 | C_j) < p(v_r^2 | C_j) < \dots< p(v_{r}^{o_r} | C_j),
    \end{equation}
    where the superscript indices $1, 2, \ldots, o_r$, now explicitly represent the ranking order of probabilities. 
    The edge with the smallest weight is selected at each step to incrementally form $ \mathcal{T}_{j,r}$ according to Eq.~(6). Since $\mathbf{W}_{j,r}(u,u+1)$ between adjacent nodes $v_r^{u}$ and $v_{r}^{u+1}$ exhibits the local minimum weight, the edge $b_r^{u,u+1}$ is included in the edge set of $\mathcal{T}_{j,r}$. Consequently, all pairs of adjacent nodes are sequentially connected, yielding a tree structure. Assuming that $s > u$, the distance between any two nodes $ v_r^{u}$ and $ v_r^{s}$ is then given by:
    \begin{equation}
    \begin{aligned}
    \mathbf{W}_{j,r}(u,s) &= \sum_{\eta=u}^{s-1} \left( p(v_{r}^{\eta+1} | C_j) - p(v_{r}^{\eta} | C_j) \right) \\
    &= |{p}({v_r^{u}}| {C_j}) - p({v_r^{s}|{C_j}})|
    \end{aligned}.
    \end{equation}
    This matches the direct edge weight in the original graph \( G_{j,r} \), demonstrating that the relation tree introduces no information loss. This preservation implies that no alternative tree structure can yield the same set of pairwise distances unless it is identical to \( \mathcal{T}_{j,r} \). Therefore, both the structure and the induced metric are uniquely determined by the conditional probability distribution, confirming the determinism of the construction.
    \end{proof}

\begin{lemma}
    \label{lm:linear}
        The distance metric induced by the relation tree $\mathcal{T}_{j,r}$ always forms a linear structure, which is naturally compatible with the one-dimensional Euclidean geometry.
    \end{lemma}
\begin{proof}
        To rigorously prove that the inferred tree $\mathcal{T}_{j,r}$ is linear, we have the following three steps: 
        
        1) Base Case: When $o_r = 2$, the statement trivially holds.
        
        2) Inductive Hypothesis: Assume the proposition holds for some arbitrary natural number \( o_r = \omega \), and reindex the nodes such that:
        \begin{equation}
            p(v_r^1 | C_j) < p(v_r^2 | C_j) < \dots< p(v_{r}^{\omega} | C_j),
        \end{equation}
        the nodes $ v_r^1, v_r^2, \dots, v_{r}^{\omega} $ form a linear structure, i.e., $v_r^u$ connects $v_r^{u+1}$ for $1< u< \omega$.

        3) Inductive Step: We now prove that the proposition holds for $ o_r = \omega +1$. Without loss of generality, assume that the probability of the newly added node $ v_{r}^{\omega+1} $, denoted as $p(v_{r}^{\omega+1} | C_j)$, satisfies $p(v_{r}^{\omega} | C_j) < p(v_{r}^{\omega+1} | C_j)$.
        In this case, the edge $ w_{j,r}^{\omega,\omega+1} $ is the smallest, and thus, according to Eq.~(6), $ v_{r}^{\omega} $ and $ v_{r}^{\omega+1} $ are connected, leading to a linear structure represented as $ v_r^1 $, $ v_r^2 $, ..., $ v_{r}^{\omega}, v_{r}^{\omega+1}$.

        Thus, the tree consisting of $\omega+1$ nodes remains linear, which is consistent with Euclidean distance in \( \mathbb{R}^1 \), allowing for seamless integration between the two. As a result, the metric induced by the relation tree is Euclidean-compatible.

    \end{proof}

According to Lemmas~\ref{lm:tree} and ~\ref{lm:linear}, we have Theorem~\ref{thm:graph_tree_consistency}. 
\begin{theorem}
\label{thm:graph_tree_consistency}
The relation tree \( \mathcal{T}_{j,r} \) inferred from the fully connected graph \( G_{j,r} \) defines a deterministic, Euclidean-compatible distance metric over the attribute values.

\end{theorem}
\begin{proof}
    It follows from Lemmas~\ref{lm:tree} and \ref{lm:linear} that $\mathcal{T}_{j,r}$ can be simplified into a linear structure with $o_r - 1$ edges connecting $o_r$ nodes without loss of information, and the resulting distance metric is both deterministic and Euclidean compatible.   
\end{proof}

\subsection{Valid Distance Metric}

\begin{lemma}\label{lm:d_i_r_g}
The distance measure $\mathbf{D}_{j,r}(u,s)$ defined within the context of the relation tree $\mathcal{T}_{j,r}$ indeed constitutes a valid distance metric.
\end{lemma}

\begin{proof}
The distance measure $\mathbf{D}_{j,r}(u,s)$ follows non-negativity, symmetry, and triangle inequality for any $r\in\{1,2,...,l\}$ and $u,s,g\in\{1,2,...,o_r\}$ as shown below.\\
1) \textbf{Non-negativity:} $\mathbf{D}_{j,r}(u,s)\geq 0 $. $\mathbf{D}_{j,r}(u,s)$ denotes a distance, which is always a non-negative weight according to Eqs.~(4) and~(7) in the submitted paper;\\
2) \textbf{Symmetry}: $\mathbf{D}_{j,r}(u,s) = \mathbf{D}_{j,r}(s,u) $. Since the relation tree is an undirected graph, the weights extracted from the undirected graph obey the commutative law for their summation;\\
3) \textbf{Triangle inequality}: $\mathbf{D}_{j,r}(u,s) \le \mathbf{D}_{j,r}(u,g) + \mathbf{D}_{j,r}(g,s)$. The shortest path between two values with length $\mathbf{D}_{j,r}(u,s)$ on the relation tree is the unique path. If we detour from $v_r^{u}$ to $v_r^{s}$ via another node $v_r^{g}$ that is not on the shortest path between nodes $v_r^{u}$ and $v_r^{s}$, it necessarily involves extra weight(s) from the other paths. Given that each weight is non-negative, the result follows.
\end{proof}

According to Lemma~\ref{lm:d_i_r_g}, we have the Theorem~\ref{thm:sample_cluster_distance}.

\begin{theorem}\label{thm:sample_cluster_distance}
The sample-cluster distance $ \Phi (\mathbf{x}_i, \mathbf{m}_j; \mathcal{T}_j) $ defined in the context of $\mathcal{T}_j$ represents a valid distance metric.
\end{theorem}

\begin{proof}
The computation of $ \Phi (\mathbf{x}_i, \mathbf{m}_j; \mathcal{T}_j) $ can be viewed as the weighted sum of a series of $\mathbf{D}_{j,r}(u,s)$ with non-negative weights represented by the probabilities in $p(v_r^{u}|C_j)$ according to Eqs.~(8) and~(9) in the submitted paper. Since distance $\mathbf{D}_{j,r}(u,s)$ is a metric according to Lemma~\ref{lm:d_i_r_g}, $ \Phi (\mathbf{x}_i, \mathbf{m}_j; \mathcal{T}_j) $ is also a metric following non-negativity, symmetry, and triangle inequality. The result follows.
\end{proof}

\subsection{Converge Guarantee}

\begin{lemma}\label{lm:inner}
The inner loop attains a local minimum of $ z(\mathbf{H}, M, \mathcal{T}) $ within a finite number of iterations when $\mathcal{T}$ is held constant.
\begin{proof}
    Given $X$ with finite numbers of samples $ n $ and clusters $ k $, the possible configurations of the partition $\mathbf{H}$ and center set $ M $ are also finite. Consequently, each distinct pair $(\mathbf{H}, M)$ can appear no more than once during the iterative process. Consider two distinct iterations $ \mathcal{I}_1 $ and $ \mathcal{I}_2 $ such that $\mathbf{H}^{\{\mathcal{I}_1\}} = \mathbf{H}^{\{\mathcal{I}_2\}}$ and $ M^{\{\mathcal{I}_1\}} = M^{\{\mathcal{I}_2\}} $. It follows that $ z(\mathbf{H}^{\{\mathcal{I}_1\}}, M^{\{\mathcal{I}_1\}}, \mathcal{T}^{\{E\}}) = z(\mathbf{H}^{\{\mathcal{I}_2\}}, M^{\{\mathcal{I}_2\}}, \mathcal{T}^{\{E\}}) $. However, the values of the objective function $\{z(\mathbf{H}^{\{1\}}, M^{\{1\}},\mathcal{T}^{\{E\}})$, $z(\mathbf{H}^{\{2\}}, M^{\{2\}},\mathcal{T}^{\{E\}})$, $\ldots$, $z(\mathbf{H}^{\{I\}}, M^{\{I\}}, \mathcal{T}^{\{E\}})\}$ produced by the inner loop exhibit strictly decreasing. Thus, the inner loop is guaranteed to converge to a local minimum in a finite number of steps.
\end{proof}
\end{lemma}
    
\begin{lemma}\label{lm:outer}
The outer loop $E$ converges to a local minimum of $z(\mathbf{H}, M, \mathcal{T})$ in a finite number of iterations.
\begin{proof}
    Note that the subspace category relationships $\mathcal{T}$ reside in a DISCrete space, implying that the number of possible configurations for $\mathcal{T}$ is finite. Additionally, the number of possible cluster partition $\mathbf{H}$ and center set $M$ is finite for $X$ consisting of $n$ samples and $k$ clusters. Therefore, the number of distinct combinations of $\mathbf{H}$, $M$, and $\mathcal{T}$ is also finite. Each possible combination of $\mathbf{H}$, $M$, and $\mathcal{T}$ appears at most once during the execution of the algorithm. Suppose that there exist two distinct iterations $E_1$ and $E_2$, where $\mathbf{H}^{\{E_1\}} = \mathbf{H}^{\{E_2\}}$, $M^{\{E_1\}} = M^{\{E_2\}}$, and $\mathcal{T}^{\{E_1\}} = \mathcal{T}^{\{E_2\}}$. The corresponding updates $\mathcal{T}^{\{E_1+1\}}$ and $\mathcal{T}^{\{E_2+1\}}$ can be computed deterministically, and $\mathcal{T}^{\{E_1+1\}} = \mathcal{T}^{\{E_2+1\}}$. Upon convergence of the subsequent inner loop $\mathcal{I}$, the resulting assignments $\mathbf{H}^{\{E_1+1\}}$, $M^{\{E_1+1\}}$ and $\mathbf{H}^{\{E_2+1\}}$, $M^{\{E_2+1\}}$ must also satisfy $\mathbf{H}^{\{E_1+1\}} = \mathbf{H}^{\{E_2+1\}}$ and $M^{\{E_1+1\}} = M^{\{E_2+1\}}$.
    Consequently, the objective function at the next iteration satisfy $z(\mathbf{H}^{\{E_1+1\}}, M^{\{E_1+1\}},  \mathcal{T}^{\{E_1+1\}}) = z(\mathbf{H}^{\{E_2+1\}}, M^{\{E_2+1\}}, \mathcal{T}^{\{E_2+1\}})$. Since the sequence $\{ z(\mathbf{H}^{\{1\}}, M^{\{1\}}, \mathcal{T}^{\{1\}}), z(\mathbf{H}^{\{2\}}, M^{\{2\}}, \mathcal{T}^{\{2\}}), \ldots,\\ z(\mathbf{H}^{\{E\}}, M^{\{E\}}, \mathcal{T}^{\{E\}}) \}$ is monotonically decreasing, the outer loop $E$ converges to a local minimum in a finite number of steps.
\end{proof}
\end{lemma}

According to Lemmas~\ref {lm:inner} and~\ref{lm:outer} we have the Theorem~\ref{the:converge}.

\begin{theorem} \label{the:converge}
DISC algorithm converges to a local minimum in a finite number of iterations.
\end{theorem}
\begin{proof} 
Based on Lemma~\ref{lm:inner} and Lemma~\ref{lm:outer}, the convergence of the inner and outer loop, i.e., lines 5 - 9 and lines 3 - 16 of Algorithm~1 is proven, respectively. Consequently, the convergence of Algorithm DISC can be firmly validated. 
\end{proof}

\subsection{Time Complexity of DISC}
\label{AP:time_complex}
\begin{theorem}
    Time complexity of DISC is $O(nlk{\mathcal{I}}{E})$.
\end{theorem}
\begin{proof}
    Assume solving problem $z(\mathbf{H},\hat{M},\hat{\mathcal{T}})$ involves ${\mathcal{I}}$ iterations to compute $\mathbf{H}$, and $z(\hat{\mathbf{H}},M,\hat{\mathcal{T}})$ involves ${\mathcal{I}}$ iterations to compute $M$. The whole algorithm involves ${E}$ iterations to infer $\mathcal{T}$ for solving $z(\hat{\mathbf{H}},\hat{M},\mathcal{T})$. For worst-case analysis, we adopt $\sigma$ to indicate the maximum number of possible values of a dataset by $\sigma=\max(o_1,o_2,...,o_l)$.

For each iteration of ${\mathcal{I}}$, $\mathbf{H}$ should cluster $n$ samples to $k$ clusters by considering $\sigma$ values of $l$ attributes according to Eq.~(2) in the submitted paper, with time complexity $O(nlk\sigma)$. $M$ computes the mode for each cluster by traversing $l$ attributes of all samples. The complexity is $O(nlk)$. Therefore, the time complexity of solving $z(\mathbf{H},\hat{M},\hat{\mathcal{T}})$ and $z(\hat{\mathbf{H}},M,\hat{\mathcal{T}})$ in a total of ${\mathbf{I}}$ iterations is $O(\mathcal{I}nlk\sigma)$.

For each iteration of ${E}$, since $\mathbf{H}$ and $M$ have been prepared, the fully connected graphs $G$ should be prepared by going through all the samples within each cluster. The complexity is $O(nlk\sigma)$. The inference process from $G$ to $\mathcal{T}$ should consider $l$ attributes with $\sigma$ possible within each cluster. The time complexity is $O(kl\sigma^{2}\log\sigma)$.
For ${E}$ iterations of the whole DISC algorithm, considering the ${\mathcal{I}}$ inner iterations, the overall time complexity of DISC is $O({E}(\mathcal{I}nlk\sigma+nlk\sigma+kl\sigma^{2}\log\sigma))$.

Since $\sigma$ is a small integer ranging from 2 to 8 in most cases, it is conventionally regarded as a constant term in asymptotic complexity analysis. Consequently, the dominant computational complexity of DISC is $O(nlk{\mathcal{I}}{E})$.
\end{proof}

\subsection{Optimization Algorithm of DISC}

The problem of optimization DISC can be formulated as computing $\mathbf{H}$, $M$, and $\mathcal{T}$ to minimize $z$, which can be summarized into the following three steps:
1) Fix $M$ and $\mathcal{T}$, compute $\mathbf{H}$; 2) Fix $\mathcal{T}$ and $\mathbf{H}$, compute $M$; 3) Repeat 1) and 2) until convergence, fix $\mathbf{H}$ and $M$, and reinfer $\mathcal{T}$ according to Eqs. (4), (5), and (7). These steps are repeated until $\mathbf{H}$ no longer changes. The whole algorithm is summarized as Algorithm ~\ref{alg:DISC}, which can be proven to guarantee convergence and has a complexity of $O(nlk{\mathcal{I}}{E})$, linear to $n$, $l$, and $k$.

\begin{algorithm}[h]
    
    \SetAlgoLined 
    
    \caption{Optimization algorithm of DISC.}
    \label{alg:DISC}
    \SetKwInOut{Input}{Input}
    \SetKwInOut{Output}{Output}
    \Input {Dataset $X$, number of sought clusters $k$.}
    \Output {Partition $\mathbf{H}$, subspace structure $\mathcal{T}$.}
    Initialization: Set outer and inner loop counters by ${E}\leftarrow 0$ and ${\mathcal{I}}\leftarrow 0$; Execute $k$-modes~\cite{kmd} to obtain initial $\mathbf{H}^{\{E\}}$ and $M$; Construct initial $\mathcal{T}^{\{E\}}$ according to $\mathbf{H}^{\{E\}}$;     
    Set convergence flag for outer loop by $Flag\_{E}\leftarrow False$.\\
    \While{$Flag\_{E}=False$}{
        Set convergence flag for inner loop by $Flag\_{\mathcal{I}}\leftarrow False$;\\
    \While{$Flag\_{\mathcal{I}}=False$}{
        ${\mathcal{I}}\leftarrow {\mathcal{I}}+1$; Compute $\mathbf{H}^{\{{\mathcal{I}}\}}$ by \textbf{Eq. (2)};\\    
        \If {$\mathbf{H}^{\{{\mathcal{I}}\}}=\mathbf{H}^{\{{\mathcal{I}}-1\}}$}{
                $Flag\_{\mathcal{I}}\leftarrow True$;
               }
               }
        \eIf {$\mathbf{H}^{\{{E}\}}=\mathbf{H}^{\{{\mathcal{I}}\}}$}{$Flag\_{E}\leftarrow True$;}{
        ${E}\leftarrow {E}+1$; $\mathbf{H}^{\{{E}\}}\leftarrow \mathbf{H}^{\{{\mathcal{I}}\}}$; Reinfer $\mathcal{T}^{\{{E}\}}$; 
        }
    }
\end{algorithm}

\end{document}